\documentclass[twoside]{article}

\usepackage[accepted]{aistats2021}
\usepackage[round]{natbib}

\usepackage[utf8]{inputenc} 
\usepackage{hyperref}       
\usepackage{url}            
\usepackage{booktabs}       
\usepackage{amsfonts}       
\usepackage{nicefrac}       
\usepackage{microtype}      
\usepackage{graphicx}
\usepackage{multirow}       
\usepackage{bm}             
\usepackage{bbm}            
\usepackage{algorithm}
\usepackage[noend]{algpseudocode}
\usepackage{amsthm}
\usepackage{amsmath}

\graphicspath{ {./figures/} }

\newtheorem{lemma}{Lemma}
\newtheorem{theorem}{Theorem}
\newtheorem*{theorem*}{Theorem}

\newcommand{\MYCOMMENT}[1]{\State \textit{\# #1}}
\newcommand{\eqnref}[1]{Equation~\ref{#1}}

\newcommand{\Sample}{\textsc{\small REGSAMP}}
\DeclareMathOperator{\Tr}{Tr}

\DeclareMathOperator{\Lang}{Lang}

\newcommand{\bigO}{\mc{O}}
\newcommand{\Z}{\mc{Z}}
\newcommand{\E}{\mc{E}}
\newcommand{\EL}{\mc{E}^{\scriptstyle \ell}}
\newcommand{\ER}{\mc{E}^{\scriptstyle r}}
\newcommand{\QL}{Q_{\ell}}
\newcommand{\QR}{Q_{r}}
\newcommand{\sR}{|s|_R}
\newcommand{\sRone}{|s_1|_{R_1}}
\newcommand{\sRtwo}{|s_2|_{R_2}}
\newcommand{\Pt}{\tilde{P}}
\newcommand{\R}{\mathbb{R}}
\newcommand{\starchar}{0}
\newcommand{\alphamat}{\alpha\alpha^T}
\newcommand{\omegamat}{\omega\omega^T}

\newcommand{\field}[1]{\R^{#1}}
\newcommand{\mc}[1]{\mathcal{#1}}
\newcommand{\seq}[2]{{#1}_{1}, {#1}_{2}, \ldots, {#1}_{#2}}

\begin{document}

\twocolumn[

\aistatstitle{Tensor Networks for Probabilistic Sequence Modeling}

\aistatsauthor{Jacob Miller \And Guillaume Rabusseau \And John Terilla}
\aistatsaddress{Mila and DIRO \\
Université de Montréal \\
\href{mailto:jmjacobmiller@gmail.com}{\texttt{jmjacobmiller@gmail.com}} \And
CCAI chair - Mila and DIRO \\
Université de Montréal \\
\href{mailto:grabus@iro.umontreal.ca}{\texttt{grabus@iro.umontreal.ca}} \And
CUNY and Tunnel \\
City University of New York \\
\href{mailto:jterilla@gc.cuny.edu}{\texttt{jterilla@gc.cuny.edu}} } ]

\begin{abstract}

Tensor networks are a powerful modeling framework developed for computational many-body physics, which have only recently been applied within machine learning. In this work we utilize a uniform matrix product state (u-MPS) model for probabilistic modeling of sequence data. We first show that u-MPS enable sequence-level parallelism, with length-$n$ sequences able to be evaluated in depth $O(\log n)$. We then introduce a novel generative algorithm giving trained u-MPS the ability to efficiently sample from a wide variety of conditional distributions, each one defined by a regular expression. Special cases of this algorithm correspond to autoregressive and fill-in-the-blank sampling, but more complex regular expressions permit the generation of richly structured data in a manner that has no direct analogue in neural generative models. Experiments on sequence modeling with synthetic and real text data show u-MPS outperforming a variety of baselines and effectively generalizing their predictions in the presence of limited data.

\end{abstract}

\section{Introduction}
\label{sec:intro}

Tensor network models have long represented the state of the art in modeling complex quantum systems~\citep{white1992, fannes1992, orus2019}, but have only recently been utilized as models for machine learning~\citep{novikov2015, cohen2016, stoudenmire2016, novikov2017, han2018, stoudenmire2018, cheng2019}. In contrast to neural networks, tensor networks forgo the use of nonlinear activation functions, relying instead on multiplicative interactions to capture complex correlations within data. This gives tensor networks a convenient mathematical structure suitable for proving general theoretical results, such as the separation in expressivity between almost all deep tensor networks and their shallow counterparts~\citep{cohen2016}. However, these distinctive mathematical properties have yet to be leveraged for the development of new \emph{operational} abilities, which would 
give more practical reasons for the wider adoption of tensor network models in real-world machine learning tasks.
 
In this work we apply a recurrent tensor network, the \emph{uniform matrix product state}~(u-MPS), to the task of probabilistic sequence modeling, and identify several novel abilities of u-MPS regarding their evaluation and generative capabilities. Despite its recurrent nature, we show that sequential inputs to u-MPS can be processed in a highly parallel manner, with sequences of length $n$ being evaluated in parallel time $\bigO(\log n)$. While the difficulty of parallelizing deep recurrent neural networks (RNNs) has previously motivated the development of non-recurrent architectures for sequence processing tasks (e.g. \citep{gehring2017, vaswani2017}), our finding shows that recurrent tensor networks represent another means of achieving greater parallelism.

We further show that u-MPS models are endowed with surprising generative capabilities closely tied to the structure of regular expressions (regex). While standard autoregressive models are constrained to generate sequences in a stream-like fashion conditioned on some prompt, we find that u-MPS can sample from a wide variety of distributions defined by conditioning regular expressions $R$. Our sampling algorithm efficiently produces unbiased samples from the probability distribution learned by the u-MPS, conditioned on the output sequence matching a given regular expression $R$. Standard autoregressive sampling follows from the choice $R = p\Sigma^*$ (for $p$ a prefix string and $\Sigma^*$ the regex matching all sequences), but other special cases include fill-in-the-blank sampling ($R = p\Sigma^*s$, for suffix $s$), as well as the generation of samples constrained to contain some target phrase $t$ ($R=\Sigma^*t\Sigma^*$).

Besides permitting the generation of sequences with rich internal structure, these techniques enable novel forms of regularization, where a u-MPS model can be penalized or incentivized during training to generate strings matching some target pattern. Such regularization can be applied in a variety of challenging tasks, including automatic code generation and mitigating gender bias in language models. Experiments on synthetic and real structured text datasets confirm these novel parallelism, sampling, and regularization benefits, and show u-MPS able to successfully generalize non-local correlations present in small strings to sequences of significantly greater length.

\paragraph{Summary of Contributions}
We give the first implementation of a u-MPS for probabilistic sequence modeling, and uncover several remarkable capabilities of this model\footnote{The code used to produce our results can be found at \url{https://github.com/jemisjoky/umps_code}.}. The absence of nonlinear activation functions in the u-MPS allows us to utilize a parallel evaluation method during training and inference. We develop new techniques linking the structure of u-MPS ``transfer operators'' to that of regular expressions, which in turn enables a flexible recursive sampling algorithm and novel forms of regularization for the u-MPS. We expect these techniques to open up significant new research directions in the design of sequential generative models, with language modeling being a particularly promising domain.

\paragraph{Related Work}
Notable previous applications of tensor networks in machine learning include compressing large neural network weights~\citep{novikov2015}, proving separations in the expressivity of deep vs shallow networks~\citep{cohen2016}, and for supervised~\citep{stoudenmire2016, novikov2017, glasser2018} and unsupervised~\citep{han2018, stoudenmire2018, cheng2019} learning tasks. Of particular relevance is~\citep{stokes2019}, where (non-uniform) MPS were trained as generative models for fixed-length binary sequences using the density matrix renormalization group (DMRG) algorithm. A diverse range of tensor network architectures have also been proposed as theoretical tools for modeling and understanding natural language, such as~\citep{pestun2017, coecke2010, gallego2017, degiuli2019}. The completely positive maps employed in our sampling algorithm are similar to those used in hidden quantum Markov models (HQMM)~\citep{monras2010,srinivasan2018}, and can likewise be interpreted using concepts from quantum information theory.

This work can be seen as a continuation of~\citep{pestun2017a}, where u-MPS were introduced from a theoretical perspective as a language model, but without the parallelization, sampling, or experimental results given here. Our sampling algorithm is a significant generalization of the fixed-length Born machine algorithm introduced in~\citep{han2018}~(which in turn follows that of~\citep{ferris2012}), and by virtue of the recurrent nature of u-MPS, permits the generation of discrete sequences of arbitrary length. The u-MPS model is equivalent to quadratic weighted finite automata~\citep{bailly2011} and (to a lesser extent) the norm-observable operator model (NOOM)~\citep{zhao2010}, and is also an example of a linear (second-order) RNN~\citep{rabusseau2019}. Benefits of linear (first-order) RNNs for parallelization and interpretability were described in~\citep{martin2018, foerster2017input}.

A key difference from previous works is the general techniques developed for evaluating and sampling from regex-structured probability distributions, which to the best of our knowledge are completely new. These techniques apply not only to u-MPS, but also to a broad family of models with similar internal structure, such as weighted finite automata (WFA)~\citep{droste2009}, hidden Markov models (HMM)~\citep{rabiner1986}, and predictive state representations~\citep{littman2002}. We consequently expect the algorithms developed here for regex sampling, regularization, and parallel evaluation to immediately generalize to any of these models which parameterize valid probability distributions.

\section{Background}

We consider sequences over a finite alphabet $\Sigma$, with $\Sigma^n$ denoting the set of all length-$n$ strings, $\Sigma^*$ the set of all strings, and $\varepsilon$ the empty string. We use $\lVert v\rVert$ to denote the 2-norm of a vector, matrix, or higher-order tensor $v$, and $\Tr(M)=\sum_{i=1}^D M_{ii}$ to denote the trace of a square matrix $M\in\field{D\times D}$.

A real-valued\footnote{The restriction to real-valued tensors is natural for machine learning, but differs from the standard in quantum physics of using complex parameters. The results given here carry over to the complex setting, and only require the replacement of some tensors by their complex conjugate.} tensor $\mc{T} \in \R^{d_1\times d_2\times\cdots\times d_n}$ is said to have shape $(\seq{d}{n})$, and can be specified by an indexed collection of elements $\mc{T}_{\seq{i}{n}} \in \R$, where each index $i_k \in [d_k] := \{1, 2, \ldots, d_k\}$. Tensors with $n$ indices are said to be $n$th order, and the set of $n$th order tensors form a vector space of dimension $\Pi_{k=1}^n d_k$.
\begin{figure}[tb]
\vskip 0.1in
\begin{center}
\centerline{\includegraphics[width=\columnwidth]{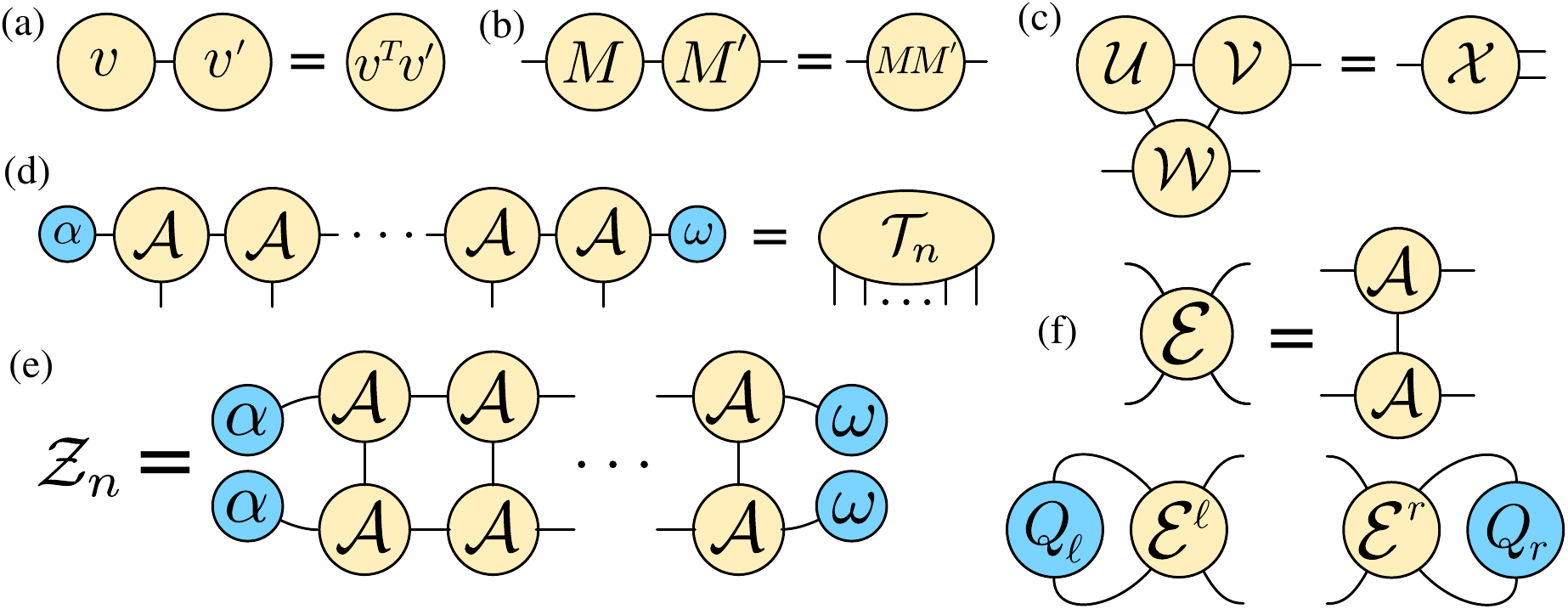}}
\caption{(a-b) Two well-known cases of tensor contractions, inner products of vectors and matrix multiplication. (c) A simple tensor network, where 2nd, 3rd, and 4th order tensors are contracted to form a 3rd order tensor. In numerical libraries, small tensor contractions can be computed with the \texttt{einsum} function, and the output $\mc{X}$ is independent of contraction order. (d) The u-MPS model, which uses a core tensor $\mc{A}$ of shape $(D, d, D)$ and $D$-dimensional vectors $\alpha$ and $\omega$ to define tensors of arbitrary order. (e) The length-$n$ normalization factor $\Z_n$ defined by \eqnref{eq:partition_n}, expressed as a network of tensor contractions. (f) The 4th order tensor $\E$ defined by two copies of the u-MPS core tensor $\mc{A}$. The contraction of $\E$ with a matrix on the left or right gives the left and right \emph{transfer operators} of the u-MPS, linear maps which allow the efficient computation of $\Z_n$ via \eqnref{eq:normalization}.}
\label{fig:contraction}
\end{center}
\end{figure}
Matrices, vectors, and scalars are the simplest examples of tensors, of 2nd, 1st, and 0th order, respectively. Tensor contraction is a generalization of both matrix multiplication and vector inner product, and multiplies two tensors along a pair of indices with equal dimension. If the tensors $\mc{T}$ and $\mc{T'}$ have respective shapes $(d_1,\ldots,d_k,\ldots,d_n)$ and $(d'_1,\ldots,d'_{k'},\ldots,d'_{n'})$, for $d_k = d'_{k'}$, then the contraction of the $k$ and $k'$ indices gives a product tensor $\mc{T''}$, described by elements
\begin{align}
\label{eq:contraction}
& T''_{i_1,\ldots,i_{k-1},i_{k+1},\ldots,i_{n},i'_1,\ldots,i'_{k'-1},i'_{k'+1},\ldots,i'_{n'}} = \nonumber\\
&\hspace{2.5cm} \sum_{i_k=1}^{d_k}\mc{T}_{i_1,\ldots,i_k,\ldots,i_{n}} \mc{T}'_{i'_1,\ldots,i_{k},\ldots,i'_{n'}}.
\end{align}
The contraction operation \eqnref{eq:contraction} is more easily understood with a convenient graphical notation~(see Figure~\ref{fig:contraction}), where individual tensors correspond to nodes in an undirected graph, and edges describe contractions to be performed. Contracting along an index corresponds to merging two connected nodes, to produce a new node whose outgoing edges are the union of those in the tensors being contracted. An important property of tensor contraction is its generalized associativity, so that a network of tensors can be contracted in any order, with the final product tensor being the same in every case.

A natural example of an $n$th order tensor is a probability distribution over length-$n$ sequences $\Sigma^n$, where the probabilities associated with all possible sequences form the $|\Sigma|^n$ separate tensor elements. This exponential growth in the number of elements makes dense representations of higher order tensors infeasible, but convenient tensor decompositions frequently permit the efficient manipulation of tensors with high order, even into the thousands.

The fixed-size matrix product state~\citep{perez2007}~(MPS, also known as tensor train~\citep{oseledets2011})~model parameterizes an $n$th order tensor $\mc{T}$ with shape $(\seq{d}{n})$ as a sequential contraction of $n$ independent tensor ``cores'' $\{\mc{A}^{(j)}\}_{j=1}^n$, which form the parameters of the model. Each $\mc{A}^{(j)}$ has shape $(D_{j-1}, d_j, D_j)$, where $D_0 = D_n = 1$. The dimensions $D_j$ are referred to as bond dimensions (or ranks) of the MPS, and by choosing the $D_j$ to be sufficiently large, it is possible to exactly represent any $n$th order tensor.

\section{Uniform MPS}
\label{sec:umps}

In this work we utilize the \emph{uniform MPS}~(u-MPS) model, a recurrent tensor network obtained by choosing all cores of an MPS to be identical tensors $\mc{A}^{(j)} = \mc{A}$ with shape $(D, d, D)$. To obtain scalar tensor elements, $D$-dimensional vectors $\alpha$ and $\omega$ are used as ``boundary conditions'' to terminate the initial and final bond dimensions of the network. In contrast to fixed-length MPS, the recurrent nature of u-MPS allows the generation of $n$th order tensors $\mc{T}_n\in\R^{d^n}$ for any $n\in\mathbb{N}$, which in turn allows u-MPS to be applied in problems involving sequential data. 

For discrete sequences over an alphabet $\Sigma$ of size $d$, a \mbox{u-MPS}\ (paired with a bijection $\varphi: \Sigma \to [d]$) can be used to map a sequence of arbitrary length-$n$ to the index of an $n$th order tensor $\mc{T}_n$, defining a scalar-valued function $f_{\mc{A}}$ over sequences. Using $\mc{A}(c) = \mc{A}_{:, \varphi(c), :} \in \R^{D\times D}$ to denote the matrix associated with the character $c\in\Sigma$, a u-MPS acts on a sequence $s = s_1s_2\cdots s_n \in \Sigma^n$ as
\begin{equation}
\label{eq:mps_amp}
f_{\mc{A}}(s) = \alpha^{T}\mc{A}(s_1)\mc{A}(s_2)\cdots \mc{A}(s_n)\omega = \alpha^{T}\mc{A}(s)\omega,
\end{equation}
where we use $\mc{A}(s) := \mc{A}(s_1)\mc{A}(s_2)\cdots \mc{A}(s_n)$ to denote the matrix product appearing in \eqnref{eq:mps_amp}. The function $\mc{A}(s)$ can be seen as a matrix-valued representation of arbitrary sequences $s\in\Sigma^*$, and is \textit{compositional} in the sense that $st$ is represented by the product of representations $\mc{A}(s)$ and $\mc{A}(t)$.

\begin{figure*}[t]
\centering
\centerline{\includegraphics[width=0.85\textwidth]{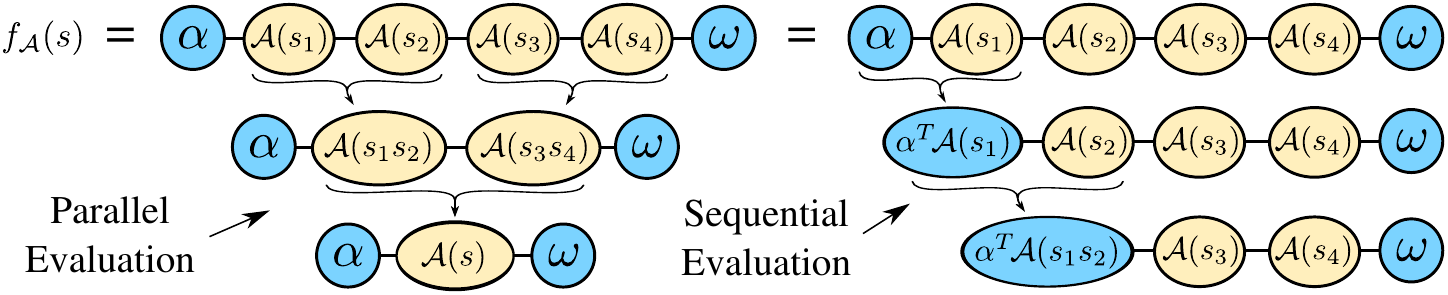}}
\caption{\small Illustration of parallel and sequential evaluation of $f_{\mc{A}}(s)$ when $|s| = 4$, where $f_{\mc{A}}(s) = (\mathcal{T}_4)_{i_1, i_2, i_3, i_4}$, an element of the 4th order tensor defined by a u-MPS. After obtaining the matrix representations $\mc{A}(s_1),\ldots,\mc{A}(s_n)$ from $s$, parallel evaluation involves repeated batch multiplications of nearest-neighbor pairs of matrices, with the boundary vectors $\alpha$ and $\omega$ only incorporated after the matrix product $\mc{A}(s)$ has been obtained. Sequential evaluation instead uses iterated matrix-vector multiplications starting with a boundary vector to contract this product. Parallel and sequential evaluation have respective costs of $\bigO(n D^3)$ and $\bigO(n D^2)$, but the former can be carried out in $\bigO(\log n)$ parallel time. The mathematical equivalence of these evaluation strategies is a basic example of the associativity of tensor contractions, allowing an appropriate method to be chosen based on the size of the model, the problem at hand, and the availability of hardware acceleration.}
\label{fig:parallel}
\end{figure*}

While u-MPS are clearly laid out as a sequential model, the evaluation of $f_{\mc{A}}(s)$ for $|s| = n$ can be parallelized by evaluating \eqnref{eq:mps_amp} using $\lceil \log_2(n) \rceil$ batched matrix-matrix multiplications on all nearest-neighbor pairs of matrices, as shown in Figure~\ref{fig:parallel}. This form of parallelization requires the absence of nonlinear activation functions in the evaluation, and can also be carried out in linear RNNs~\citep{martin2018}.

\subsection{Born Machines}

While \eqnref{eq:mps_amp} is identical to the evaluation rule for WFA, and well-suited for regression tasks, we are interested in using u-MPS as probabilistic models. This requires the interpretation of $f_{\mc{A}}(s)$ as a non-negative probability $P(s)$, and deciding if a general WFA outputs negative values is  undecidable~\citep{denis2008}. This issue can be circumvented by requiring all entries of $\mc{A}$, $\alpha$, and $\omega$ to be non-negative real numbers, but such models can be seen as largely equivalent to hidden Markov models~\citep{denis2008}.

We instead follow the approach introduced in~\citep{pestun2017a} (see also~\citep{han2018}), which is inspired by the typical usage of MPS in quantum mechanics. For the case of u-MPS, this \emph{Born machine} approach converts a scalar value $f_{\mc{A}}(s)$ to an unnormalized probability $\Pt(s) := \left|f_{\mc{A}}(s)\right|^2$. This can be converted into a properly normalized distribution over sequence of fixed length $n$ by choosing $P_n(s) = \Pt(s) / \Z_n$, where the normalization function $\Z_n$ is given by
\begin{align}
\label{eq:partition_n}
\Z_n &= \sum_{s \in \Sigma_n} \Pt(s) = \sum_{i_1\in[d]} \sum_{i_2\in[d]} \cdots \sum_{i_n\in[d]} \left| (T_n)_{i_1,i_2,\ldots,i_d} \right|^2 \nonumber\\ 
&= \left\lVert\mc{T}_n\right\rVert^2,
\end{align}
and with $\mc{T}_n$ the $n$th order tensor defined by the u-MPS. This quadratic evaluation rule is equivalent to the Born rule of quantum mechanics~\citep{born1926}, which gives a formal interpretation of such models as wavefunctions over $n$ quantum spins. However this probabilistic correspondence is richer in the case of u-MPS, since distributions over sequences of different lengths can be easily defined. The distribution $P(s) = \Pt(s) / \Z_*$ in particular gives a probability distribution over strings of arbitrary length, where the normalization factor $\Z_*$ is identical to that given in \eqnref{eq:partition_n}, but with the sum over $\Sigma^n$ replaced by one over $\Sigma^*$. We show in Section~\ref{sec:regex} how normalization functions of this form can be generalized further to incorporate sums over all strings matching an arbitrary regular expression $R$.

Normalization functions like $\Z_n$ occur frequently in many-body physics, and can be efficiently computed via a simple reordering of tensor contractions. By \eqnref{eq:partition_n}, $\Z_n$ equals the 2-norm of $\mc{T}_n$, which is represented diagrammatically as Figure~\ref{fig:contraction}e. The naive method of evaluating $\Z_n$ involves first generating all elements of $\mc{T}_n$ via contraction along the horizontal $D$-dimensional indices of the u-MPS, but the generalized associativity of tensor contraction lets us evaluate this expression more efficiently.

By first contracting two copies of $\mc{A}$ along a vertical $d$-dimensional index (see~\eqnref{fig:contraction}f) we obtain a 4th order tensor $\E$, which can be interpreted as a linear map on a space of matrices in two main ways, by contracting either its left or its right indices with an input. These linear maps, known as \emph{transfer operators}, are examples of completely positive (CP) maps, a generalization of stochastic matrices which find frequent application in the context of quantum information theory (see Appendix~\ref{sec:appendix_a} for more details). These maps admit the Kraus representations $\ER(\QR) = \sum_{c\in\Sigma} \mc{A}(c) \QR \mc{A}(c)^T$ and $\EL(\QL) = \sum_{c\in\Sigma} \mc{A}(c)^T \QL \mc{A}(c)$, which are connected by the adjoint identity $\Tr(\QL \ER(\QR)) = \Tr(\EL(\QL) \QR)$.\footnote{In general, CP maps are linear operators $\mc{F}$ acting on square matrices by the rule $\mc{F}(Q) = \sum_{i=1}^K A_i Q A_i^T$. CP maps are guaranteed to send positive semidefinite (PSD) to other PSD matrices, allowing us to assume in the following that all $\QL$ and $\QR$ are PSD.}

The normalization $\Z_n$ can be equivalently computed in terms of left or right transfer operators, with the latter option yielding
\begin{align}
\label{eq:normalization}
\Z_n &= \alpha^T \ER(\ER(\cdots\ER(\omegamat))\cdots) \alpha \nonumber\\
         &= \Tr\left( \QL^\alpha\ (\ER)^{\circ n}(\QR^\omega) \right),
\end{align}
where $\QL^\alpha = \alphamat$ and $\QR^\omega = \omegamat$ are rank-1 matrices constituting boundary conditions for the normalization term. We use $(\ER)^{\circ n}$ to denote the composition of $\ER$ with itself $n$ times, and define $(\ER)^{\circ 0}$ to be the identity map acting on square matrices. For an MPS of bond dimension $D$ over an alphabet of size $d$, a single transfer operator application requires time $\bigO(dD^3)$, giving a sequential runtime of $\bigO(ndD^3)$ for computing $\Z_n$. By representing transfer operators as $D^2 \times D^2$ matrices, this computation can be parallelized in a similar manner as described in Section~\ref{sec:umps}, but at the price of increasing the total computational cost to $\bigO(nD^6)$.

The distribution $P(s) = \Pt(s) / \Z_*$ over all strings $s\in\Sigma^*$ plays an important role in the following, but to employ this we must ensure the infinite summation $\Z_* = \sum_{s\in\Sigma^*} \Pt(s)$ does in fact converge. This convergence can be guaranteed by rescaling the core tensor to a new $\mc{A}' = \gamma \mc{A}$, for any scalar $\gamma$ satisfying $0 < |\gamma| < \sqrt{\rho(\ER)}$, where $\rho(\ER)$ is the spectral radius of $\ER$. Such rescaling leaves all fixed-length distributions $P_n(s)$ invariant, while introducing a bias towards shorter ($|\gamma| < 1$) or longer ($|\gamma| > 1$) strings in $P(s)$. 

\section{Regular Expressions and u-MPS}
\label{sec:regex}
\begin{table}[t]
\setlength\aboverulesep{0.1pt}
\setlength\belowrulesep{0.1pt}
\renewcommand{\arraystretch}{1.7}
\caption{Dictionary giving the correspondence between regular expressions (regex) and generalized transfer operators associated with a u-MPS (note the reversal of order in $\EL_{R_1 R_2}$). The infinite sum giving $\ER_{S^*}(\QR)$ can be efficiently approximated by partial sums, or else computed as the solution $\QR^*$ to the linear equation $(I - \ER_S)\QR^* = \QR$ (similarly for $\EL_{S^*}(\QL)$).}
\label{tab:regex_cp}
\begin{center}
\begin{small}
\begin{sc}
\begin{tabular}{c|c|c}
    \textbf{Regex} $\mathbf{R}$ & $\bm{\ER_R(\QR)}$ & $\bm{\EL_R(\QL)}$ \\
    \midrule
    $\mathbf{s}$ &         $\mc{A}_s \QR \mc{A}_s^T$                             & $\mc{A}_s^T \QL \mc{A}_s$ \\
    $\mathbf{R_1 R_2}$ &   $\ER_{R_1}(\ER_{R_2}(\QR))$                           & $\EL_{R_2}(\EL_{R_1}(\QL))$ \\
    $\mathbf{R_1 | R_2}$ & $\ER_{R_1}(\QR) + \ER_{R_2}(\QR)$                     & $\EL_{R_1}(\QL) + \EL_{R_2}(\QL)$ \\
    $\mathbf{S^*}$ &       $\sum_{n=0}^\infty (\ER_S)^{\circ n}(\QR)$ & $\sum_{n=0}^\infty (\EL_S)^{\circ n}(\QL)$ 
\end{tabular}
\end{sc}
\end{small}
\end{center}
\vskip -0.1in
\end{table}
While transfer operators as defined above are standard in quantum many-body physics, we now show how this transfer operator calculus can be richly generalized in the setting of sequential data. We work with regular expressions (regex) $R$ over an alphabet $\Sigma$ of size $d$, which can be recursively defined in terms of: (a) String literals $s\in\Sigma^*$, (b) Concatenations of regex $R = R_1 R_2$, (c) Unions of regex $R = R_1 | R_2$, and (d) Kleene closures of regex $R = S^*$. We use $\Sigma$ to denote the union regex of all single characters $c \in \Sigma$, and $\Sigma^n$ to denote the concatenation of $\Sigma$ with itself $n$ times.

Any regex $R$ defines a set $\Lang(R) \subset \Sigma^*$, the language of strings matching the pattern specified by $R$. While $\Lang(R)$ is uniquely determined by $R$, it is typically possible to choose multiple regex which define the same language. We assume in the following that we have chosen an unambiguous regex $R$, so that each string $s \in \Lang(R)$ matches $R$ exactly once. This involves no loss of generality, since any ambiguous regex can be replaced by an unambiguous regex defining the same language~\citep{book1971}. In such cases, we will use $R$ to also represent the subset $\Lang(R)$.

To each regex $R$, we associate a pair of generalized transfer operators $\ER_{R}$ and $\EL_{R}$, formed by summing over all strings in the language $R$, which act on matrices as
\begin{align}
\label{eq:transfer_op}
\ER_R(\QR) &= \sum_{s\in R} \mc{A}(s) \QR \mc{A}(s)^T, \nonumber\\ 
\EL_R(\QL) &= \sum_{s\in R} \mc{A}(s)^T \QL \mc{A}(s).
\end{align}
While the naive sum appearing in~\eqnref{eq:transfer_op} can have infinitely many terms, the action of such CP maps can still be efficiently and exactly computed in terms of the recursive definition of the regex itself. Table~\ref{tab:regex_cp} gives the correspondence between the four primitive regex operations introduced above and the corresponding operations on CP maps. Proof of the consistency between these recursive operations and \eqnref{eq:transfer_op} for unambiguous regex, as well as a generalized correspondence holding for arbitrary regex, is given in Appendix~\ref{sec:appendix_a}.

While most regex operations in Table~\ref{tab:regex_cp} are straightforward, the Kleene closure $\ER_S$ involves an infinite summation which is guaranteed to converge whenever the spectral norm of $\ER_S$ is bounded as $\rho(\ER_S) < 1$. We denote the value of this convergent sum by $\QR^*$, which can be approximated using a finite number of summands or alternately computed as the solution to the linear equation $(I - \ER_S)\QR^* = \QR$ (see~\citep{balle2019}).

A fruitful way of interpreting the transfer operators $\ER_R$ and $\EL_R$ is as normalization functions for u-MPS sampling distributions. We define the quantity $\Z_R(\QL, \QR) = \Tr(\QL \ER_R(\QR))$ to be the (unnormalized) probability associated to a regex $R$ in the presence of boundary matrices $\QL, \QR$, and the quantity $\Z_R = \Z_R(\alphamat, \omegamat)$ as utilizing the boundary matrices of the u-MPS. It follows from \eqnref{eq:transfer_op} that $\Z_R = \sum_{s \in R}\Pt(s)$ does indeed give the unnormalized probability associated to all strings $s$ matching $R$. We recover as special cases of this the quantities $\Z_n = \Z_{\Sigma^n}$ and $\Z_* = \Z_{\Sigma^*}$ defined above.

\section{Regex Sampling and Regularization}
\label{sec:sampling}
\begin{algorithm}[tb]
\caption{Regex sampling algorithm for u-MPS}
\label{alg:sampling}
\begin{algorithmic}
\Function{\Sample}{$R, Q_\ell, Q_r$}
  \If{$R = s$} \hspace{4.0em} 
    \MYCOMMENT{Sample a string literal $s \in \Sigma^*$}
    \State\Return{$s$}
  \ElsIf{$R = R_1 R_2$} \hspace{0.0em} 
    \MYCOMMENT{Sample a sequence of expressions}
    \State $s_1 = \Sample(R_1, \QL, \ER_{R_2}(\QR))$
    \State $s_2 = \Sample(R_2, \EL_{s_1}(\QL), \QR)$
    \State\Return{$s_1 s_2$}
  \ElsIf{$R = R_1 | R_2$} \hspace{-0.2em} 
    \MYCOMMENT{Sample a union of expressions}
    \State Sample random $i \in \{1, 2\}$, with probs 
    \State \hspace{0.7cm} $p(i) = \Z_{R_i}(\QL, \QR) \ / \ \Z_{R_1 | R_2}(\QL, \QR)$
    \State $s_i = \Sample(e_i, \QL, \QR)$
    \State\Return{$s_i$}
  \ElsIf{$R = S^*$} \hspace{1.0em} 
    \MYCOMMENT{Sample regex $S$ zero or more times}
    \State Sample random $i \in \{\textrm{HALT}, \textrm{GO}\}$, with probs 
    \State \hspace{0.7cm} $p(\textrm{HALT}) = \Tr(\QL \QR) / \Z_{S^*}(\QL, \QR)$, 
    \State \hspace{0.7cm} $p(\textrm{GO}) = 1 - p(\textrm{HALT})$
    \If{$i = \textrm{HALT}$} \hspace{0.0em} \MYCOMMENT{Return empty string}
      \State\Return{$\varepsilon$}
    \Else \hspace{6.0em} \MYCOMMENT{Sample one or more chars}
      \State\Return{$\Sample(SS^*, \QL, \QR)$}
    \EndIf
  \EndIf
\EndFunction
\end{algorithmic}
\end{algorithm}
The correspondence developed above between syntactic operations on regex and linear-algebraic operations on transfer operators endows u-MPS models with surprising capabilities unavailable to probabilistic models based on neural networks. We discuss the application of these techniques for sampling from conditional distributions of strings matching a target regex, as well as for utilizing a novel form of task-specific regularization during training.

\subsection{Sampling}

We introduce a regex-parameterized sampling function $\Sample$ in Algorithm~\ref{alg:sampling}. $\Sample$ gives a recursive means of converting any regex $R$ into an efficient sampling procedure, whose random outputs are (for unambiguous $R$) unbiased samples from the conditional u-MPS distribution associated with the subset $R \subset \Sigma^*$. This is formalized in
\begin{theorem}
\label{thm:sample}
Consider a u-MPS model with core tensor $\mc{A}$ and boundary vectors $\alpha$ and $\omega$, along with an unambiguous regex $R$ whose right transfer operator $\ER_R$ converges. Let $P$ indicate the probability distribution over arbitrary strings defined by the u-MPS, so that $\Sigma_{s\in\Sigma^*} P(s) = 1$. Then calling $\Sample(R, \alphamat, \omegamat)$ samples a string $s \in \Sigma^*$ from the conditional u-MPS distribution $P(s | s \in R) = P(s) / P(R)$, with $s \in R$ and where $P(R) := \sum_{s' \in R} P(s')$.
\end{theorem}
We prove Theorem~\ref{thm:sample} in Appendix~\ref{sec:appendix_b}, which also discusses the use of ambiguous regex $R$. For this latter case, Algorithm~\ref{alg:sampling} works identically, but weights strings $s$ by the number of times $s$ matches $R$.

Although Algorithm~\ref{alg:sampling} is written in a recursive manner, it is useful to consider the simple example $R = \Sigma^n$, a concatenation of the single-character regex $\Sigma$ with itself $n$ times, to understand the overall control flow. In this case, Algorithm~\ref{alg:sampling} first attempts to sample the initial character in the string via a recursive call to $\Sample(\Sigma, \alphamat, \ER_{\Sigma^{n-1}}(\omegamat))$. This requires $n-1$ applications of the transfer operator $\ER$ to the initial right boundary matrix, and yields one new character before continuing to the right and repeating this process again.

As is common with recursive algorithms, caching intermediate information permits the naive cost of $(n-1) + (n-2) + \cdots + 1 = \bigO(n^2)$ transfer operator applications to be reduced to $\bigO(n)$. This cached version is equivalent to a simple iterative algorithm, where a sequence of right boundary matrices is first generated and saved during a right-to-left sweep, before a left-to-right sweep is used to sample text and propagate conditional information using the left boundary matrices. Using this idea, we show in Appendix~\ref{sec:appendix_c} that for typical regex $R$, Algorithm~\ref{alg:sampling} can be run with average-case runtime $\bigO(L d D^3)$ and worst-case memory usage $\bigO(L D^2)$, for $L$ the number of characters in $R$, $d$ the size of $\Sigma$, and $D$ the bond dimension of the u-MPS.

\subsection{Regularization}

The normalization function $\Z_{R}$ defined by a regex $R$ gives the unnormalized probability assigned to all strings matching $R$. We first show that for the case of unambiguous $R$, this probability can be properly normalized.
\begin{theorem}
\label{thm:regularization}
Consider a u-MPS model and an unambiguous regex $R$ satisfying identical conditions as in Theorem \ref{thm:sample}. Then the probability $P(R) = \sum_{s \in R} P(s)$ assigned by the u-MPS to the set of strings matching $R$ can be exactly calculated as $P(R) = \Z_R / \Z_*$.
\end{theorem}
The practical importance of Theorem~\ref{thm:regularization} lies in the ability to compute any $\Z_R = \Tr(\alphamat \ER_R(\omegamat))$ inside of an automatic differentiation library, possibly with the aid of techniques described in~\citep{liao2019}. By making $P(R)$ directly computable as a function of the u-MPS parameters $\mc{A}$, $\alpha$, and $\omega$, this probability can be incorporated as a regularizer (i.e. a differentiable loss term) during training.

Although it is not immediately clear how to think about such ``regex regularizers'', we provide three examples which can be used during gradient-based training of a u-MPS model. First, $P(R)$ can be directly added to the loss, encouraging gradient updates of the model to minimize the probability of strings belonging to $R$. In the context of language models, this could be used to avoid learning offensive phrases seen in training data, for example by choosing $R = \Sigma^* S \Sigma^*$ with $S$ a union of strings extracted from a dataset of abusive language.

A related loss is $\mc{L} = |P(R_1) - P(R_2)|$, which penalizes differences in the probabilities assigned to regex $R_1$ and $R_2$. Such regularization would be most effective when the regex $R_1$ and $R_2$ are similarly constructed, with the loss enforcing an indifference between these two options. This could be applied for the mitigation of gender bias in language models, for example by choosing each $R_i$ to be $R_i = \Sigma^* s_i \Sigma^*$, for $s_1$ and $s_2$ a pair of identical but oppositely gendered phrases (e.g. ``his career'' vs. ``her career'' or ``he cooks'' vs. ``she cooks''). 

Finally, using a loss $\mc{L} = -\log(P(R))$ encourages maximizing the probability of strings belonging to $R$. This type of regularization is natural when all strings produced by the model should belong to some regular language, for example when choosing a syntactically valid variable name in a code completion task. The extension of Theorem~\ref{thm:sample} and \ref{thm:regularization} to \emph{context-free} languages would greatly broaden the range of applications for these methods in language modeling, given the fundamental role played by context-free grammars in structuring natural language.

\section{Experiments}
\label{sec:experiments}
\begin{table}[t]
\caption{Experiments on Tomita grammars 3-7 (see Appendix~\ref{sec:appendix_d} for the definitions of these grammars), where the training data is randomly chosen from strings of lengths between 1 and 15 belonging to the grammar. The trained models are used to sample strings of lengths 16 and 30, with the percentage of grammatically correct samples reported. The u-MPS consistently gives better generalization across different lengths (quite substantially for Tomita 5), except for Tomita 6 which neither model is able to learn. Most of the Tomita grammars are too small to train with more than 1,000 strings, but Tomita 5 and 6 permit experiments with larger datasets.}
\label{tab:tomita}
\setlength\aboverulesep{0.1pt}
\setlength\belowrulesep{0.1pt}
\renewcommand{\arraystretch}{1.1}
\begin{center}
\begin{small}
\begin{sc}
\begin{tabular}{lc|cccc}
    Tomita & Samp. & \multirow{2}{*}{u-MPS} & \multirow{2}{*}{HMM} & \multirow{2}{*}{LSTM} & \multirow{2}{*}{TR} \\
    ($N_{\mathrm{train}}$) & Len. &  &  &  \\
    \midrule
    3 (1K)  & 16 & \textbf{100.0} & 91.6          & 90.2          & 28.8 \\
    3 (1K)  & 30 & \textbf{100.0} & 82.0          & 85.6          & 9.4  \\
    4 (1K)  & 16 & \textbf{99.9}  & 99.4          & 85.4          & 50.7 \\
    4 (1K)  & 30 & 99.5           & \textbf{99.6} & 64.7          & 32.5 \\
    5 (10K) & 16 & \textbf{100.0} & 52.0          & 49.9          & 51.1 \\
    5 (10K) & 30 & \textbf{99.9}  & 49.8          & 52.8          & 50.5 \\
    6 (10K) & 16 & \textbf{35.9}  & 34.4          & 33.1          & 32.9 \\
    6 (10K) & 30 & 33.1           & 33.1          & \textbf{34.4} & 32.7 \\
    7 (1K)  & 16 & \textbf{99.3}  & 98.1          & 89.2          & 51.3 \\
    7 (1K)  & 30 & \textbf{89.4}  & 79.3          & 29.1          & 10.0 \\
    
\end{tabular}
\end{sc}
\end{small}
\end{center}
\vskip -0.1in
\end{table}
\begin{table}[t]
\caption{Experiments on the context-free Motzkin grammar, where the training set is fixed to contain only strings of length 15. We explore both fixed-length sampling (Samp) and character completion (Comp) tasks, where the model either samples a string from scratch, or predicts a missing character in a reference string given access to the character's prefix and suffix. In each case, the same trained u-MPS, HMM, and Transformer are used to generate both sampling and character completion data. The bidirectional LSTM performs best on shorter strings in the character completion task, but quickly degrades in accuracy as the length of the reference strings are increased.}
\label{tab:motzkin}
\setlength\aboverulesep{0.1pt}
\setlength\belowrulesep{0.1pt}
\renewcommand{\arraystretch}{1.1}
\begin{center}
\begin{small}
\begin{sc}
\begin{tabular}{lc|cccc}
    Task & Str. & \multirow{2}{*}{u-MPS} & \multirow{2}{*}{HMM} & \multirow{2}{*}{LSTM} & \multirow{2}{*}{TR} \\
    ($N_{\mathrm{train}}$) & Len. &      &  &  &  \\
    \midrule
    Samp (1K)  & 1  & \textbf{89.4} & 37.4 & 41.7           & 39.1 \\
    Samp (1K)  & 16 & \textbf{74.4} & 30.3 & 41.2           & 2.3  \\
    Samp (1K)  & 50 & \textbf{32.5} & 12.6 & 0.0            & 0.6  \\
    Samp (10K) & 1  & \textbf{99.3} & 36.0 & 35.7           & 36.2 \\
    Samp (10K) & 16 & \textbf{99.8} & 34.3 & 60.4           & 0.5  \\
    Samp (10K) & 50 & \textbf{91.6} & 12.4 & 5.4            & 0.2  \\
    Comp (1K)  & 1  & 89.4          & 39.2 & \textbf{99.9}  & 32.1 \\
    Comp (1K)  & 16 & 69.6          & 29.1 & \textbf{99.5}  & 30.7 \\
    Comp (1K)  & 50 & 58.8          & 13.1 & \textbf{61.3}  & 30.2 \\
    Comp (10K) & 1  & 99.3          & 36.3 & \textbf{100.0} & 33.5 \\
    Comp (10K) & 16 & 99.8          & 31.7 & \textbf{100.0} & 34.5 \\
    Comp (10K) & 50 & \textbf{92.4} & 14.8 & 69.1           & 33.7 \\
    
\end{tabular}
\end{sc}
\end{small}
\end{center}
\vskip -0.1in
\end{table}
\begin{table}[t]
\caption{Runtimes for computing the loss and gradient with respect to model parameters using a u-MPS with bond dimension 50, for a batch of 100 strings of length 500 evaluated on a CPU or GPU. While computation on a CPU favors sequential evaluation, owing to its lower overall cost, the reduced parallel depth inherent to parallel evaluation leads to a reduced runtime in the presence of GPU hardware acceleration.}
\label{tab:runtime}
\renewcommand{\arraystretch}{1.1}
\begin{center}
\begin{small}
\begin{sc}
\begin{tabular}{c|cc}
          & Sequential Eval. & Parallel Eval.   \\
          & Runtime (ms)     & Runtime (ms) \\
\hline
CPU comp. & \textbf{73.0}    & 232.7 \\
GPU comp. & 49.9             & \textbf{45.2}  \\
\end{tabular}
\end{sc}
\end{small}
\end{center}
\vskip -0.1in
\end{table}

We use experiments on synthetic and real structured text datasets to assess the performance of u-MPS in probabilistic sequence modeling and grammatical inference, as well as to verify the benefits of u-MPS for parallelization, regex sampling, and regularization. 

\subsection{Synthetic Experiments}

We first carry out experiments on several synthetic text datasets consisting of five Tomita grammars of binary strings and a context-free ``Motzkin'' grammar over the trinary alphabet $\Sigma_{M} = \{\mathtt{\ (\ ,\ \starchar\ ,\ )\ }\}$~\citep{tomita1982, alexander2018}. The latter consists of all strings whose parentheses are properly balanced, with no constraints placed on the $\mathtt{\starchar}$ characters. 

In each case we train the u-MPS on strings of a restricted length from the grammar and then sample new strings of unseen lengths from the trained u-MPS, with the model assessed on the percentage of sampled strings which match the grammar. The sampling comes in two forms, either fixed length-$n$ sampling (corresponding to $R = \Sigma^n$), or character completion sampling, where a single character in a reference string is masked and the prefix and suffix $p$ and $s$ are used to guess it (corresponding to $R = p \Sigma s$). While more general sampling experiments can easily be imagined, we have chosen these tasks because they allow for direct comparisons with a variety of baselines, including (unidirectional and bidirectional) LSTMs, HMMs, and Transformers.

While unbiased fixed-length sampling is easy for u-MPS via Algorithm~\ref{alg:sampling}, we found that the unidirectional LSTM baseline required an additional positional encoding in its inputs to avoid rapid degeneration in the output text when sampling past the longest length seen in training. At sampling time, we vary the length scale associated with this encoding based on the desired sampling length, so that the final step of sampling is always associated with the same positional encoding vector.

We train the u-MPS, LSTM, HMM, and small Transformer models using gradient descent on a negative log likelihood (NLL) loss with the Adam~\citep{kingma2015} optimizer. For each experiment we use models of $D=20$ and $D=50$ hidden units, with LSTMs chosen to have one layer and Transformers with two layers and 4 heads. Five independent trials are used for each value of $D$, with the final validation loss used to select the best model for generating samples. We use a piecewise constant learning rate between $10^{-2}$ and $10^{-4}$, and early stopping to choose the end of training.

In the Tomita experiments~(Table~\ref{tab:tomita}) u-MPS give impressive performance, in many cases achieving perfect accuracy in sampling strings of unseen sizes within the language. This is true not only in the simpler grammars Tomita 3 and 4, but also in the more difficult Tomita 5, where valid strings satisfy the nonlocal constraint of containing an even number of \texttt{0}'s and of \texttt{1}'s. The HMM and LSTM also attain reasonably high sample accuracy, although in a manner that degrades faster with sequence length than the u-MPS, while the Transformer performs worst.

Similar results are seen with the context-free Motzkin language~(Table~\ref{tab:motzkin}), where a fixed-length sampling task similar to the Tomita experiment is paired with a character completion task. A separate bidirectional LSTM must be used for this latter task, since unidirectional LSTMs cannot make use of future context information. By contrast, a trained u-MPS model can be employed in both of these settings without any task-specific adaptation, as well as in more general sentence completion tasks involving connected or disjoint regions of missing text~(tasks which cannot be easily handled by standard RNN models). The u-MPS does substantially better in generalizing and reproducing the structure of Motzkin strings than the unidirectional LSTM, HMM, and Transformer, being able to sample strings of over 3 times the length seen in training with over 90\% accuracy. The u-MPS is outperformed only by the bidirectional LSTM in character completion experiments on smaller training sets.

Given the ability of HMMs to exactly reproduce the distributions associated with Tomita languages and (bounded length) Motzkin languages, it is surprising that the u-MPS still manages to more easily learn such distributions in practice. Somewhat surprising also is the poor performance of the Transformer models, which is likely a result of the small sizes of the string datasets used.

We additionally benchmark the relative runtime of sequential and parallel evaluation for u-MPS running on a CPU or GPU (Table~\ref{tab:runtime}). For a u-MPS of $D$ hidden states with strings of length $n$, RNN-style sequential evaluation has parallel depth $\mathcal{O}(n)$ and cost $\mathcal{O}(nD^2)$, while parallel evaluation trades this for a parallel depth of $\mathcal{O}(\log n)$ and cost $\mathcal{O}(nD^3)$. This would suggest parallel evaluation having benefits for the total runtime when hardware acceleration is present, which the runtimes in Table~\ref{tab:runtime} confirm.

\subsection{Email Experiments}
\begin{table}[t]
\caption{Training on emails and assessing the per-character perplexity (PPL) and accuracy of unconditional sampling using a u-MPS with bond dimension 50. Making use of a regularizer associated with a regex $R_e$, approximating the formatting of valid email addresses, leads to small gains in both the syntactic correctness of unconditionally sampled strings and overall perplexity. Note that the use of \emph{conditional} sampling relative to $R_e$ would guarantee the generation of syntactically valid strings (Theorem~\ref{thm:sample}), a fact we verify experimentally.}
\label{tab:regularization}
\renewcommand{\arraystretch}{1.1}
\begin{center}
\begin{small}
\begin{sc}
\begin{tabular}{c|cc}
           & u-MPS          & u-MPS   \\
           & w/ regex reg.  & w/o regex reg. \\
\hline
Correct \% & \textbf{37.2}  & 35.4    \\
Test PPL   & \textbf{7.3}   & 7.8     \\
\end{tabular}
\end{sc}
\end{small}
\end{center}
\vskip -0.1in
\end{table}

To verify the correctness and relative benefits of regex sampling and regularization, we train on a dataset of approximately 4,000 email addresses taken from the CLAIR fraudulent emails dataset~\citep{radev2008}. Although the correctness of an email address generally depends on non-syntactic considerations (such as domain name resolution), we can approximate the format of valid email addresses using the regex $ R_{e} = \verb![\w-.]+@([\w-]+.)+[\w-][\w-]+! $ (a generalization of the pattern \texttt{name@site.tld}).

We first train a u-MPS using gradient descent on this email address dataset and look at the perplexity of a held-out validation set and correctness of (unconditionally) sampled strings, both with and without the use of regularization associated with $R_e$. We find the use of regex regularization to yield small improvements to perplexity and correctness of sampled text, as shown in Table~\ref{tab:regularization}. 

Although the correctness of conditional regex sampling is already guaranteed by Theorem~\ref{thm:sample}, we confirm this experimentally by using the regex $R_e$ to periodically produce conditional samples during training. We find that conditional sampling relative to $R_e$ does indeed always yield random strings matching the desired regex, but with the quality of generated text gradually improving during training. While conditional sampling of the randomly initialized u-MPS produces syntactically-valid but otherwise random strings, such as {\verb k90@4riuh2600xz1.wz }, training the model on the email address dataset leads to the production of more realistic-looking email addresses, such as {\verb sail203@yahoo.com }.

\section{Conclusion}

We develop a u-MPS model for probabilistic modeling of sequence data, which we show has distinctive capabilities regarding parallelism, sampling, and regularization, in a manner which mirrors the structure of regular languages. Although our results are derived in the specific context of u-MPS, the underlying techniques used to demonstrate these capabilities are applicable to other models with a similar mathematical layout, such as WFA, HMM, and PSR. As a result, our Theorems~\ref{thm:sample} and \ref{thm:regularization} should generalize to other such models. We expect the algorithms developed here to be associated with different runtimes when applied to other models, leading to different performance tradeoffs of parallelization, sampling, and regularization methods than what is reported here. For example, the parallel evaluation method used here requires $\bigO(n D^3)$ resources for a u-MPS, compared with $\bigO(n D^6)$ for a HQMM. 

Beyond these immediate generalizations, a more interesting extension of our results is the generalization of Theorems~\ref{thm:sample} and \ref{thm:regularization} to the setting of context-free languages. While all regular languages are context-free, the latter has significantly greater expressive power and relevance for applications in natural language processing and automatic code generation. Surprisingly, we have found that such languages can indeed be sampled from and employed for regularizers within u-MPS models, although with a higher cost of $\bigO(D^6)$.

Given the ability we demonstrate to use grammars to constrain the probability distribution of u-MPS models, a natural question is whether the inverse process is possible: Namely, given a trained u-MPS model, do there exist automatic means of identifying grammatical rules which account for some portion of the correlations present within the learned probability distribution? Such techniques would represent a qualitatively new type of grammatical inference, where grammar rules and language models interact in a two-way manner. 

Finally, a natural next step is scaling up u-MPS for real-world sequence modeling tasks, notably language modeling. Some current obstacles to this process are (a) the $\bigO(D^3)$ cost of certain u-MPS operations, and (b) the absence of well-established best-practices for training large tensor networks with gradient descent. We expect these obstacles to be overcome by dedicated engineering effort and the rapidly growing number of software libraries for manipulating tensor networks, along with the adoption of powerful computational methods developed by the many-body physics community into machine learning. Considering the unexpected benefits demonstrated here, we expect recurrent tensor network architectures to have a bright future. 

\section*{Acknowledgements}
This research is supported by the Canadian Institute for Advanced Research (CIFAR AI chair program).

\appendix
\onecolumn

\section{Completely Positive Maps and Generalized Transfer Operators}
\label{sec:appendix_a}

In this section, we give definitions and known results concerning completely positive (CP) maps and regular expressions (regex), as well as further details regarding the assignment of generalized transfer operators to regex\footnote{We present all definitions and results in this section in terms of real-valued matrices, but the corresponding statements for complex-valued matrices and CP maps is obtained by replacing matrix transposes $Q^T$ by Hermitian adjoints $Q^\dagger$, representing the transposed and complex-conjugated counterpart of $Q$.}. We conclude with a proof that the recursive definition of generalized transfer operators given in Table~\ref{tab:regex_cp}
(repeated here for convenience) has an equivalent representation in terms of a weighted sum over all strings in the regex, which for unambiguous regex gives precisely the simple form of Equation~\ref{eq:transfer_op}.

We say that a matrix $Q\in\field{D\times D}$ is positive semidefinite (PSD) if it is~(a) Symmetric ($Q^T = Q$), and~(b) satisfies $v^T Q v \geq 0$ for every $v\in\field{D}$. If $Q$ further satisfies the property that $v^T Q v = 0$ only when $v = 0$, then we call it positive definite. Given a PSD matrix $Q$, the diagonal elements of $Q$ will necessarily be non-negative. For any vector $v$, the rank-1 matrix $vv^T$ is necessarily a PSD matrix, and all PSD matrices $Q\in\field{D\times D}$ can be expressed as the weighted sum of~(at most) $D$ such rank-1 matrices. This can be used to show that for any PSD matrices $Q$ and $Q'$, $\Tr(QQ') \geq0$.

It is common in quantum mechanics to regard PSD matrices as a generalized form of probabilistic states~\citep{nielsen2002}, a viewpoint which allows us to consider the matrices $\QL, \QR$ as probabilistic latent states of a u-MPS. To this end, the family of completely positive~(CP) maps is the natural generalization of stochastic maps, which act on these PSD matrices. A map $\E$ sending PSD $Q\in\field{D\times D}$ to $Q'=\E(Q)\in\field{D'\times D'}$ is said to be CP if it admits a Kraus representation, consisting of $r \geq 1$ matrices $A_i\in\field{D'\times D}$ such that $\E$ can be expressed as:
\begin{equation}
\label{eq:kraus}
\E(Q) = \sum_{i=1}^r A_i Q A_i^T
\end{equation}
The condition~\eqref{eq:kraus} implies in particular that $\E(Q)$ is PSD if $Q$ is. The matrices in \eqref{eq:kraus} are called the Kraus operators of the map, and the same CP map $\E$ can be given multiple Kraus representations with inequivalent values of $r$. The minimum value of $r$ such that $\E$ can be represented as~\eqref{eq:kraus} is called the rank of $\E$, and is always bounded as $r \leq D^2$. Nonetheless, Kraus representations with a greater number of Kraus operators can be useful for understanding the action of the map, as we will see below.

By taking the transpose of all the Kraus operators appearing in~\eqref{eq:kraus}, we obtain a new CP map $\E^T$, which is the adjoint of $\E$. Mathematically, this means that for any CP map $\E$ and positive matrices $\QL, \QR$, the equality $\Tr(\QL \E(\QR)) = \Tr(\E^T(\QL) \QR)$. For greater clarity, in the context of sequence modeling with u-MPS we frequently refer to a CP map and its adjoint as ``right'' and ``left'' maps $\ER$ and $\EL$, rather than $\E$ and $\E^T$.

The term ``transfer operator'' is common in many-body physics, and in our setting refers to a CP map $\E$ obtained from the Kraus representation with $A_i = \mc{A}(\varphi^{-1}(i))$, for $\mc{A}$ a u-MPS core tensor and $\varphi: \Sigma \to [d]$ a bijection mapping characters $c$ in the size-$d$ alphabet $\Sigma$ to the numbers $\{1, 2, \ldots, d\}$ (see Figure~\ref{fig:contraction}f). 
In Section~\ref{sec:regex} 
we introduced a generalization of this standard notion of transfer operator to include a number of other CP maps $\E_R$ associated with an arbitrary regular expression $R$, whose recursive definition is given in Table~\ref{tab:regex_cp}. 
Using $\Sigma$ to also denote the regex matching any single character in our alphabet, the standard transfer operator emerges as the special case $R = \Sigma$.

We sometimes assume that our regex $R$ is unambiguous, in the sense that any string $s$ matching $R$ matches in exactly one way. This assumption can be made without loss of generality, since any ambiguous regex $R$ can be converted into an unambiguous regex $R'$ accepting the same set of strings~\citep{book1971}. For example, the ambiguous regex $R = a^*a^*$ matches the string $s = a$ in two ways, but can be replaced with the equivalent unambiguous regex $R' = a^*$. 
\begin{table}[t]
\begin{center}
\begin{small}
\begin{sc}
\begin{tabular}{c|cccc}
    \textbf{Regex} $\mathbf{R =\ }$ & $c$ & $R_1 R_2$ & $R_1 | R_2$ & $S^*$ \\
    \midrule
    $\bm{\ER_R(\QR)}\ = $ & $\mc{A}(c) \QR \mc{A}(c)^T$ & $\ER_{R_1}(\ER_{R_2}(\QR))$ & $\ER_{R_1}(\QR) + \ER_{R_2}(\QR)$ & $\sum_{n=0}^\infty (\ER_S)^{\circ n}(\QR)$ \\
    $\bm{\EL_R(\QL)}\ = $ & $\mc{A}(c)^T \QL \mc{A}(c)$ & $\EL_{R_2}(\EL_{R_1}(\QL))$ & $\EL_{R_1}(\QL) + \EL_{R_2}(\QL)$ & $\sum_{n=0}^\infty (\EL_S)^{\circ n}(\QL)$
\end{tabular}
\end{sc}
\end{small}
\end{center}
\vskip -0.1in
\end{table}

Any regex $R$ can be built inductively from (a) Single characters $c\in\Sigma$, (b) Concatenations of regex $R = R_1 R_2$, (c) Unions of regex $R = R_1 | R_2$, and (d) Kleene closures of regex $R = S^*$. We prove by induction over the structure of $R$ that any generalized right transfer operator $\ER_R$ defined by the recursive procedure in Table~\ref{tab:regex_cp} 
acts according to a generalization of Equation~\ref{eq:transfer_op}, 
as stated in

\begin{theorem}
\label{thm:transfer_op}
Consider the generalized transfer operators $\ER_R$ and $\EL_R$ associated with an arbitrary regex $R$ and a u-MPS with core tensor $\mc{A}$, which are defined by the recursive rules in Table~\ref{tab:regex_cp}. 
Then $\ER_R$ converges if and only if $\ER_L$ converges, and in this case the transfer operators are described by the Kraus representations,
\begin{equation}
\label{eq:gen_transfer_op}
\ER_R(\QR) = \sum_{s\in \Sigma^*} \sR \mc{A}(s) \QR \mc{A}(s)^T, \qquad \EL_R(\QL) = \sum_{s\in \Sigma^*} \sR \mc{A}(s)^T \QL \mc{A}(s),
\end{equation}
where $\sR$ denotes the number of times the string $s$ matches the regex $R$. For unambiguous regex, this Kraus representation is identical to that of Equation~\eqref{eq:transfer_op}.
\end{theorem}

\begin{proof}

For each of the four types of regex $R$, we make the inductive assumption that the subexpressions of $R$ (if any) satisfy~\eqref{eq:gen_transfer_op}, and use this to prove that the transfer operator $\ER_R$ satisfies~\eqref{eq:gen_transfer_op}. This allows us to immediately prove the corresponding statement for the left transfer operator $\EL_R$.

\paragraph{$\mathbf{R = c:}$} Apparent from Table~\ref{tab:regex_cp} 
and the single string which matches $R$, $s = c$.
   
\paragraph{$\mathbf{R = R_1 R_2:}$} Assume $R_1$ and $R_2$ both satisfy~\eqref{eq:gen_transfer_op}. Table~\ref{tab:regex_cp} 
gives:
\begin{align*}
    \ER_{R_1 R_2}(\QR) &= \ER_{R_1}(\ER_{R_2}(\QR)) = \sum_{s_1 \in \Sigma^*} \sum_{s_2 \in \Sigma^*} \sRone \sRtwo \mc{A}(s_1) \mc{A}(s_2) \QR \mc{A}(s_2)^T \mc{A}(s_1)^T \\
    &= \sum_{s_1 \in \Sigma^*} \sum_{s_2 \in \Sigma^*} \sRone \sRtwo \mc{A}(s_1 s_2) \QR \mc{A}(s_1 s_2)^T = \sum_{s \in \Sigma^*} |s|_{R_1\! R_2} \mc{A}(s) \QR \mc{A}(s)^T.
\end{align*}
In the second-to-last equality we used the compositional property $\mc{A}(s_1)\mc{A}(s_2) = \mc{A}(s)$, while in the last equality we used the identity $|s|_{R_1 R_2} = \sum_{s_1 s_2 = s} |s_1|_{R_1} |s_2|_{R_2}$, where the sum over $s_1, s_2$ represents all possible partitions of $s$ into a prefix and suffix.

\paragraph{$\mathbf{R = R_1 | R_2:}$} Assume $R_1$ and $R_2$ both satisfy~\eqref{eq:gen_transfer_op}. Table~\ref{tab:regex_cp} 
gives:
\begin{align*}
    \ER_{R_1 | R_2}(\QR) &= \ER_{R_1}(\QR) + \ER_{R_2}(\QR) = \left(\sum_{s \in \Sigma^*} |s|_{R_1} \mc{A}(s) \QR \mc{A}(s)^T\right) + \left(\sum_{s \in \Sigma^*} |s|_{R_2} \mc{A}(s) \QR \mc{A}(s)^T\right) \\
    &= \sum_{s \in \Sigma^*} |s|_{R_1 | R_2} \mc{A}(s) \QR \mc{A}(s)^T.
\end{align*}
In the final equality we have used the identity $|s|_{R_1 | R_2} = |s|_{R_1} + |s|_{R_2}$.
   
\paragraph{$\mathbf{R = S^*:}$} Assume $S$ satisfies~\eqref{eq:gen_transfer_op}. The operator $\ER_{S^*}$ will converge only when the spectral norm of $\ER_S$ satisfies $\rho(\ER_S) < 1$, in which case Table~\ref{tab:regex_cp} 
gives:
\begin{align*}
    \ER_{S^*}(\QR) &= \sum_{n=0}^\infty (\ER_{S})^{\circ n}(\QR) = \sum_{n=0}^\infty \sum_{s_1\cdots s_n \in \Sigma^*} |s_1|_{S} \cdots |s_n|_S\, \mc{A}(s_1 \cdots s_n) \QR \mc{A}(s_1 \cdots s_n)^T \\
    &= \sum_{s \in \Sigma^*} |s|_{S^*}\, \mc{A}(s) \QR \mc{A}(s)^T.
\end{align*}
In the final equality we have used the identity $|s|_{S^*} = \sum_{n=0}^\infty \sum_{s_1 \cdots s_n = s} |s_1|_{S} \cdots |s_n|_{S}$, where the sum over $s_1, \ldots, s_n$ represents all possible partitions of $s$ into $n$ contiguous pieces. 

While we only characterized the action of the right transfer operators $\ER_R$, substituting all matrices $\mc{A}(s)$ with their transposed counterparts $\mc{A}(s)^T$ immediately yields the corresponding characterization for the action of $\EL_R$. In this latter case, the direction-reversing identity $\mc{A}(s_1 s_2)^T = \mc{A}(s_2)^T \mc{A}(s_1)^T$ is accounted for by the transfer operator correspondence $\EL_{R_1 R_2} = \EL_{R_2}\EL_{R_1}$. 

Because $\ER_R$ and $\EL_R$ are adjoints of each other, their eigenvalue spectra are identical, and therefore $\ER_R$ converges if and only if $\EL_R$ converges. Finally for unambiguous regex $R$, the quantity $\sR \in \{0, 1\}$, giving the equality $\sum_{s\in \Sigma^*} \sR \mc{A}(s) \QR \mc{A}(s)^T = \sum_{s\in R} \mc{A}(s) \QR \mc{A}(s)^T$ which proves Equation~\ref{eq:transfer_op}.
   
\end{proof}

Although it is not obvious a priori when a transfer operator $\ER_R$ will converge for a given regex $R$ and core tensor $\mc{A}$, it is clear that the Kleene closure is the only operation permitting divergence. Consequently, a regex $R$ will converge only when all of its subexpressions of the form $S_i^*$ have spectral norm $\rho(\ER_{S_i}) < 1$. Note that any $S^*$ for which $S$ accepts the empty string is guaranteed to produce a divergent transfer operator $\ER_{S^*}$, so that in particular any transfer operator of the form $\ER_{(S^*)^*}$ is divergent.

\section{Proof of Theorem~\ref{thm:sample}}
\label{sec:appendix_b}

In order to prove Theorem~\ref{thm:sample}, 
we first prove a more general Lemma~\ref{lem:sample}, which characterizes the probability distribution $P_R(s, \QL, \QR)$ of strings output by $\Sample(R, \QL, \QR)$ for arbitrary $R$.
\begin{lemma}
\label{lem:sample}
Consider a u-MPS model with core tensor $\mc{A}$ and a regex $R$ for which the generalized right transfer operator $\ER_R$ defined recursively by Table~\ref{tab:regex_cp} 
converges. Then for any PSD matrices $\QL, \QR$, the probability distribution of strings output by $\Sample(R, \QL, \QR)$ is $P_R(s, \QL, \QR) = \sR \Pt(s, \QL, \QR) / \mc{Z}_R(\QL, \QR)$, where $\Pt(s, \QL, \QR) = \Tr(\QL \ER_s(\QR))$, $\mc{Z}_R(\QL, \QR) = \Tr(\QL \ER_R(\QR))$, and $\sR$ counts the number of times the string $s$ matches the regex $R$.
\end{lemma}

\begin{proof}

We prove Lemma~\ref{lem:sample} by induction over the structure of $R$, where we assume that sampling from a regex subexpression $R'$ of $R$ with any boundary matrices $\QL'$ and $\QR'$ produces strings from the distribution $P_{R'}(s, \QL', \QR')$. For each of the four cases of regex formation, we use this inductive assumption to show that sampling from $R$ produces strings from the distribution $P_R(s, \QL, \QR)$.

\paragraph{$\mathbf{R = c:}$} From Algorithm~\ref{alg:sampling},
$\Sample(c, \QL, \QR)$ will always output the string $s = c$. Because the quantity $|s|_c$ is 1 when $s = c$ and 0 otherwise, the sampling distribution can be written as
\begin{align*}
    P_c(s, \QL, \QR) &= |s|_c = |s|_c \frac{\Tr(\QL \ER_s(\QR))}{\Tr(\QL \ER_c(\QR))} = |s|_c \frac{\Pt(s, \QL, \QR)}{\mc{Z}_c(\QL, \QR)}.
\end{align*}
   
\paragraph{$\mathbf{R = R_1 R_2:}$} From Algorithm~\ref{alg:sampling},
$\Sample(R_1 R_2, \QL, \QR)$ will first output a string $s_1$ from $\Sample(R_1, \QL, \ER_{R_2}(\QR))$, then use $s_1$ to output a string $s_2$ from $\Sample(R_2, \EL_{s_1}(\QL), \QR)$. Using our inductive assumption for $R_1$ and $R_2$, the probability assigned to the output string $s$ from all possible partitions into a prefix and suffix as $s_1 s_2 = s$ is
\begin{align*}
    P_{R_1 R_2}(s, \QL, \QR) &= \sum_{s_1 s_2 = s} P_{R_1}(s_1, \QL, \ER_{R_2}(\QR)) \cdot P_{R_2}(s_2, \EL_{s_1}(\QL), \QR) \\
    &= \sum_{s_1 s_2 = s} \left( \sRone \frac{\Tr(\QL \ER_{s_1}(\ER_{R_2}(\QR)))}{\Tr(\QL \ER_{R_1}(\ER_{R_2}(\QR)))} \right) \left( \sRtwo \frac{\Tr(\EL_{s_1}(\QL) \ER_{s_2}(\QR))}{\Tr(\EL_{s_1}(\QL) \ER_{R_2}(\QR))} \right) \\
    &= \sum_{s_1 s_2 = s} \sRone \sRtwo \left(\frac{\Tr(\QL \ER_{s_1 R_2}(\QR))}{\Tr(\QL \ER_{R_1 R_2}(\QR))}\right) \left(\frac{\Tr(\QL \ER_{s_1 s_2}(\QR))}{\Tr(\QL \ER_{s_1 R_2}(\QR))}\right) \\
    &= |s|_{R_1 R_2} \frac{\Tr(\QL \ER_{s}(\QR))}{\Tr(\QL \ER_{R_1 R_2}(\QR))} = |s|_{R_1 R_2} \frac{\Pt(s, \QL, \QR)}{\mc{Z}_{R_1 R_2}(\QL, \QR)}.
\end{align*}
In the third equality above, we use the composition rule $\ER_{s'}\ER_{R'} = \ER_{s' R'}$ and the adjunction rule $\Tr(\EL_{s'}(\QL') \QR') = \Tr(\QL' \ER_{s'}(\QR'))$, while in the fourth equality, we use the identity $|s|_{R_1 R_2} = \sum_{s_1 s_2 = s} |s_1|_{R_1} |s_2|_{R_2}$.

\paragraph{$\mathbf{R = R_1 | R_2:}$} From Algorithm~\ref{alg:sampling},
$\Sample(R_1 | R_2, \QL, \QR)$ will first pick a random index $i \in {1, 2}$ with probability $p(i) = \mc{Z}_{R_i}(\QL, \QR) \ / \ \mc{Z}_{R_1 | R_2}(\QL, \QR)$, and then use this to output a string $s$ from $\Sample(R_i, \QL, \QR)$. Using our inductive assumption, the probability assigned to the output string $s$ is
\begin{align*}
    P_{R_1 | R_2}(s, \QL, \QR) &= \sum_{i \in \{1, 2\}} p(i) \cdot P_{R_i}(s_i, \QL, \QR) \\
    &= \sum_{i \in \{1, 2\}} \left( \frac{\mc{Z}_{R_i}(\QL, \QR)}{\mc{Z}_{R_1 | R_2}(\QL, \QR)} \right) \left( |s|_{R_i} \frac{\Pt(s, \QL, \QR)}{\mc{Z}_{R_i}(\QL, \QR)} \right) \\
    &= |s|_{R_1 | R_2} \frac{\Pt(s, \QL, \QR)}{\mc{Z}_{R_1 | R_2}(\QL, \QR)}.
\end{align*}
In the final equality, we have used the identity $|s|_{R_1 | R_2} = |s|_{R_1} + |s|_{R_2}$.

\paragraph{$\mathbf{R = S^*:}$} To sample from the regex $R$ we must have the infinite sum defining $\ER_R$ in Table~\ref{tab:regex_cp} 
converge, which is guaranteed by the assumptions of Lemma~\ref{lem:sample}. Given this convergence, calling $\Sample(S^*, \QL, \QR)$ will either output the empty string $s = \varepsilon$, or else call $\Sample(SS^*, \QL, \QR)$. In the latter case, the concatenation rule will then sample some $s' \in S$ before calling $\Sample(S^*, \EL_{s'}(\QL), \QR)$ to sample some random number $n \geq 0$ occurrences of $S$.

We denote the unnormalized collection of probabilities associated with strings produced from exactly $n$ occurrences of $S$ as $P_{S^*}^{(n)}$, and we use an inductive proof to show that $P_{S^*}^{(n)} = p(n) P_{S^n}$, for $p(n, \QL, \QR) = \mc{Z}_{S^n}(\QL, \QR) / \mc{Z}_{S^*}(\QL, \QR)$. In other words, our recursive sampling procedure for $S^*$ is equivalent to first sampling a random length using $p(n)$, then calling the corresponding $\Sample(S^n, \QL, \QR)$.

\paragraph{\textbf{Base case} $n = 0$: } The regex $S^0$ matches only the empty string $s = \varepsilon$, and from Algorithm~\ref{alg:sampling},
this occurs with probability
\begin{equation*}
    P_{S^*}^{(0)}(s, \QL, \QR) = \frac{\Tr(\QL \QR)}{\mc{Z}_{S^*}(\QL, \QR)} = \frac{\mc{Z}_{S^0}(\QL, \QR)}{\mc{Z}_{S^*}(\QL, \QR)} = p(0) P_{S^0}(s, \QL, \QR),
\end{equation*}
where we have used the identity $\ER_{S^0} = \ER_\varepsilon = I$, and the fact that $P_{S^0}(s)$ is 1 for $s = \varepsilon$ and 0 otherwise. 

\paragraph{\textbf{Step case} $n + 1$: } From Algorithm~\ref{alg:sampling}, 
the probability of sampling a string $s = s_1 s_2$ with $s_1$ matching $S$ and $s_2$ matching $S^{n}$ is
\begin{align*}
    P_{S^*}^{(n + 1)}\!(s, \QL, \QR) &= \sum_{s_1 s_2 = s} \left(1 - \frac{\Tr(\QL \QR)}{\mc{Z}_{S^*}(\QL, \QR)} \right) P_{S}(s_1, \QL, \ER_{S^*}(\QR)) P_{S^*}^{(n)}(s_2, \EL_{s_1}(\QL), \QR) \\ 
    &= \sum_{s_1 s_2 = s} \left(\frac{\mc{Z}_{SS^*}(\QL, \QR)}{\mc{Z}_{S^*}(\QL, \QR)} \right) \left(|s_1|_{S} \frac{\Pt(s_1, \QL, \ER_{S^*}(\QR))}{\mc{Z}_S(\QL, \ER_{S^*}(\QR))}\right) P_{S^*}^{(n)}(s_2, \EL_{s_1}(\QL), \QR) \\
    &= \sum_{s_1 s_2 = s} |s_1|_{S} \frac{\Pt(s_1, \QL, \ER_{S^*}(\QR))}{\mc{Z}_{S^*}(\QL, \QR)} \frac{\mc{Z}_{S^n}(\EL_{s_1}(\QL), \QR)}{\mc{Z}_{S^*}(\EL_{s_1}(\QL), \QR)} |s_2|_{S^n} \frac{\Pt(s_2, \EL_{s_1}(\QL), \QR)}{\mc{Z}_{S^n}(\EL_{s_1}(\QL), \QR)} \\
    &= \sum_{s_1 s_2 = s} |s_1|_{S} |s_2|_{S^n} \frac{\Pt(s_2, \EL_{s_1}(\QL), \QR)}{\mc{Z}_{S^*}(\QL, \QR)} = \sum_{s_1 s_2 = s} |s_1|_{S} |s_2|_{S^n} \frac{\Pt(s, \QL, \QR)}{\mc{Z}_{S^*}(\QL, \QR)} \\
    &= \frac{\mc{Z}_{S^{n+1}}(\QL, \QR)}{\mc{Z}_{S^*}(\QL, \QR)} |s|_{S^{n+1}} \frac{\Pt(s_1 s_2, \QL, \QR)}{\mc{Z}_{S^{n+1}}(\QL, \QR)} = p(n+1, \QL, \QR) P_{S^{n+1}}(s, \QL, \QR). \\
\end{align*}
In the above we used the following identities: $\mc{Z}_{S^*}(\QL, \QR) = \Tr(\QL \QR) + \mc{Z}_{SS^*}(\QL, \QR)$ (second equality), $\mc{Z}_S(\QL, \ER_{S^*}(\QR)) = \mc{Z}_{SS^*}(\QL, \QR)$ (third equality), $\Pt(s_1, \QL, \ER_{S^*}(\QR)) = \mc{Z}_{S^*}(\EL_{s_1}(\QL), \QR)$ (fourth equality), $\Pt(s_2, \EL_{s_1}(\QL), \QR) = \Pt(s_1 s_2, \QL, \QR)$ (fifth equality), and $|s|_{S^{n+1}} = \sum_{s_1 s_2 = s} |s_1|_{S} |s_2|_{S^n}$ (sixth equality).

With this inductive characterization of the unnormalized distributions $P_{S^*}^{(n)}$, we can finally show
\begin{align*}
    P_{S^*}(s, \QL, \QR) &= \sum_{n = 0}^\infty P_{S^*}^{(n)}(s, \QL, \QR) = \sum_{n = 0}^\infty \frac{\mc{Z}_{S^{n}}(\QL, \QR)}{\mc{Z}_{S^*}(\QL, \QR)} |s|_{S^{n}} \frac{\Pt(s, \QL, \QR)}{\mc{Z}_{S^{n}}(\QL, \QR)} \\
    &= |s|_{S^*} \frac{\Pt(s, \QL, \QR)}{\mc{Z}_{S^*}(\QL, \QR)},
\end{align*}

where we have used the identity $\sum_{n = 0}^\infty |s|_{S^{n}} = |s|_{S^*}$ in the last equality.

\end{proof}

Having proved Lemma~\ref{lem:sample}, we can now prove Theorem~\ref{thm:sample} 
as a simple corollary, which is restated here for ease of reference.

\begin{theorem*}
Consider a u-MPS model with core tensor $\mc{A}$ and boundary vectors $\alpha$ and $\omega$, along with an unambiguous regex $R$ whose right transfer operator $\ER_R$ converges. Let $P_*$ indicate the probability distribution over arbitrary strings defined by the u-MPS, so that $\Sigma_{s\in\Sigma^*} P_*(s) = 1$. Then calling $\Sample(R, \alphamat, \omegamat)$ generates a random string $s \in \Sigma^*$ from the conditional u-MPS distribution $P_*(s | s \in R) = P_*(s) / P_*(R)$, where $P_*(R) := \sum_{s' \in R} P_*(s')$.
\end{theorem*}

\begin{proof}

From Lemma~\ref{lem:sample}, we know the probability distribution of strings output by $\Sample(R, \alphamat, \omegamat)$, as well as $P_*(s) = \Sample(\Sigma^*, \alphamat, \omegamat)$. This characterization lets us show
\begin{align*}
    P_R(s, \alphamat, \omegamat) &= \sR \frac{\Pt(s, \alphamat, \omegamat)}{\mc{Z}_R(\alphamat, \omegamat)} = \sR \frac{\Pt(s, \alphamat, \omegamat)}{\Tr(\alphamat \ER_R(\omegamat))} \\
    &= \sR \left( \frac{\Pt(s, \alphamat, \omegamat)}{\mc{Z}_{\Sigma^*}(\alphamat, \omegamat)} \right) \left( \frac{\mc{Z}_{\Sigma^*}(\alphamat, \omegamat)}{\sum_{s' \in R} \Pt(s', \alphamat, \omegamat)} \right) \\
    &= \sR \frac{P_{\Sigma^*}(s, \alphamat, \omegamat)}{\sum_{s' \in R} P_{\Sigma^*}(s', \alphamat, \omegamat)} = 
    \begin{cases}
    P_*(s) / P_*(R), &\text{if } s \in R\\
    0, &\text{otherwise}
    \end{cases} \\
    &= P_*(s | s \in R).
\end{align*}

In the third equality we used \eqref{eq:gen_transfer_op} in Theorem~\ref{thm:transfer_op} (which reduces to Equation~\eqref{eq:transfer_op} 
for an unambiguous $R$) to express $\ER_R$ as $\ER_R = \sum_{s \in R} \ER_s$, while also introducing a normalization factor associated with $\Sigma^*$ to the numerator and denominator. In the last equality we have used the definition of the conditional probability distribution associated with $s$ matching the regex $R$, and have also utilized the fact that $\sR$ is either 1 or 0 for unambiguous regex.

\end{proof}

\section{Runtime Analysis}
\label{sec:appendix_c}

Algorithm~\ref{alg:sampling} 
\ is written as a recursive procedure, which makes its runtime analysis nontrivial. We show here that this sampling procedure can in most cases be carried out using compute and storage costs which scale linearly with the length $L_R$ of the defining regex $R$. As a technical assumption, we require the star height of $R$ to be bounded, where the star height $h_*(R)$ is defined recursively as $h_*(c) = 0$, $h_*(R_1 R_2) = h_*(R_1 | R_2) = \max(h_*(R_1), h_*(R_2))$, and $h_*(S^*) = 1 + h_*(S)$. In practice this assumption is very mild.

\begin{theorem}
\label{thm:runtime}

Consider a core tensor $\mc{A}$ and regex $R$ of length $L_R$ with bounded star height, for which the associated right transfer operator $\ER_R$ converges. Then calling $\Sample(R, \QL, \QR)$ will return a random string of mean length $\langle n \rangle = \bigO(L_R)$, with average-case compute cost of $C_R = \bigO(L_R d D^3)$ and worst-case memory cost of $M_R = \bigO(L_R D^2)$.

\end{theorem}

\begin{proof}

We again utilize a proof by induction, with the additional assumption that $C_R$ is also an upper bound on the expected cost of applying the transfer operator $\ER_R$. For the regex length $L_R$, we consider the characters $|$, $($, $)$, $^*$, and $c$ (for $c \in \Sigma$) as having length 1, along with the single-character regex $\Sigma$. This definition of length is closer to that of real-world regex, where the metacharacter ``\texttt{.}'' corresponds to our $\Sigma$.

In order to utilize caching in Algorithm~\ref{alg:sampling}, 
we replace the simple recursive case of binary regex concatenation $R = R_1 R_2$ with a maximal concatenation of smaller regex $R = R_1 R_2 \cdots R_K$, where each $R_i$ is either a single character, a union of regex, or a Kleene closure. Such concatenations are the only place where caching is utilized, allowing us to bound the memory usage in terms of the regex length $L_R$, rather than the random string length $n_R$. We don't include the non-cache memory usage required to hold our u-MPS parameters and intermediate variables, which is in every case $\bigO(d D^2)$.

Given that the length $n_R$ of an output sample is typically a random variable, we first show that the mean length is bounded as $\langle n_R \rangle = \bigO(L_R)$. It is apparent that for any regex $R$ constructed without Kleene closures, we have the stronger bound $n_R \leq L_R$. We therefore start our inductive proof with this last remaining case of Kleene closures, showing that $\langle n_R \rangle = \bigO(L_R)$ for all regex $R$ with bounded star height. We then use this result to characterize the compute and memory requirements of Algorithm~\ref{alg:sampling}.

\paragraph{$\mathbf{R = S^*:}$} In order for $\ER_{S^*}$ to converge, we must have the spectral radius of $\ER_S$ be $\lambda_S := \rho(\ER_S) < 1$. Noting that this implies $\Tr(\QL \ER_S(\QR)) \leq \lambda_S \Tr(\QL \QR)$ for any boundary matrices $\QL, \QR$, we find that the probability of obtaining $m$ occurrences of $S$ is upper bounded as
\begin{equation*}
p(m) = \Tr(\QL (\ER_S)^{\circ m}(\QR)) / \mc{Z}_R(\QL, \QR) \leq \lambda_S^m \Tr(\QL \QR) / \mc{Z}_R(\QL, \QR) = \lambda_S^m\, p(0).
\end{equation*}
Given this exponentially decaying upper bound, the output of $\Sample(S^*)$ on average will consist of $\langle m \rangle = \bigO(\chi_S)$ calls to $\Sample(S)$, where $\chi_S := \lambda_S^{-1}$. Assuming we can obtain a boundary-independent upper bound on the expected length of $\Sample(S)$, then this proves that $\langle n_{S^*} \rangle = \bigO(\chi_S \langle n_S \rangle)$.

If $S$ itself contains expressions with deeply nested Kleene closures then this task becomes difficult, with the above bound translating to $\langle n_R \rangle = \bigO(\chi^{h_*(R)} L_R)$, for $h_*(R)$ the star height of $R$ and $\chi$ the maximum $\chi_{S_i}$ among all nested subexpressions $(S_i)^*$ within $R$. However, if we assume $R$ has bounded star height, then this reduces to $\langle n_R \rangle = \bigO(\chi^{h_*(R)} L_R) = \bigO(L_R)$, our desired bound.

Moving on to a consideration of the resource costs of $\Sample(S^*)$, we make the inductive assumption that a single call to $\Sample(S)$ has average-case runtime of $\bigO(L_S d D^3)$ and worst-case memory usage of $\bigO(L_S D^2)$. Algorithm~\ref{alg:sampling} 
in this case will $m$ samples from $S$, using the same right boundary condition $\QR^*$ each time. This leads to regex length, runtime, and memory usage of
\begin{align*}
    L_R &= L_S + 1 = \bigO(L_S), \qquad C_R = \bigO(\langle m \rangle C_S) = \bigO(\chi_S L_S d D^3) = \bigO(L_R d D^3), \\
    M_R &= M_S = \bigO(L_S D^2) = \bigO(L_R D^2),
\end{align*}
where the last equality of $C_R$ uses the bounded star height of $R$. We finally note that the above bound on $C_R$ also applies to the transfer operator $\ER_{S^*}$, whose action on $\QR$ can be approximated to arbitrary precision $\epsilon = \exp(\bigO(-m / \chi_S))$ using $m$ applications of $\ER_S$.

\paragraph{$\mathbf{R = \sigma}$\textbf{, for }$\mathbf{\sigma=c}$\textbf{ or }$\mathbf{\Sigma:}$} For the case of $R = c$, no resources are required for sampling. For $R = \Sigma$, the sampling procedure costs $C_\Sigma = \bigO(d D^3)$, which also gives an upper bound on the cost of applying the transfer operator $\ER_\sigma$. The regex and output string lengths are both 1 here and no caching is involved, so
\begin{equation*}
    L_\sigma = 1, \qquad C_\sigma = \bigO(d D^3) = \bigO(L_\sigma d D^3), \qquad M_\sigma = 0 = \bigO(L_\sigma D^2).
\end{equation*}
\paragraph{$\mathbf{R = R_1 R_2 \cdots R_K:}$} When evaluating $\Sample(R_1 \cdots R_K, \QL, \QR)$, we first compute and cache the $K$ right boundary matrices $\QR^{(1)}, \QR^{(2)}, \ldots, \QR^{(K)}$ in a right-to-left sweep, via the rules $\QR^{(K)} = \QR$ and $\QR^{(i-1)} = \ER_{R_i}(\QR^{(i)})$. This has a memory cost of $M_R = \bigO(K D^2)$. With these right boundary matrices in hand, we then use a left-to-right sweep to obtain strings $s_1, s_2, \ldots, s_K$ via repeated calls to $\Sample(R_i, \QL^{(i)}, \QR^{(i)})$, where the left boundary matrices are defined as $\QL^{(1)} = \QL$ and $\QL^{(i+1)} = \EL_{s_i}(\QL^{(i)})$. No caching of the $\QL^{(i)}$ is required, and each call to $\Sample(R_i, \QL^{(i)}, \QR^{(i)})$ generally involves some additional memory usage, which is freed immediately afterwards. 

Applying our inductive assumption about the runtime and memory usage of each of the transfer operators and $\Sample$ calls for the subexpressions $R_1, \ldots, R_K$, we get
\begin{align*}
    L_R &= \sum_{i=1}^K L_{R_i}, \qquad C_R = \bigO\!\left(\sum_{i=1}^K C_{R_i}\right) = \bigO\!\left(\sum_{i=1}^K L_{R_i} d D^3\right) = \bigO(L_R d D^3), \\
    M_R &= \bigO(K D^2) + \max_{i}(M_{R_i}) = \bigO(K D^2) + \bigO(\max_i(L_{R_i})D^2) = \bigO(L_R D^2).
\end{align*}
\paragraph{$\mathbf{R = R_1 | R_2 | \cdots | R_K:}$} To evaluate $\Sample(R_1 | \cdots | R_K, \QL, \QR)$, we must first sample a random $i$ from the distribution $p(i) = \mc{Z}_{R_i}(\QL, \QR) / \mc{Z}_R(\QL, \QR)$, then call $\Sample(R_i, \QL, \QR)$. This gives the following characterization of the overall runtime and memory usage
\begin{align*}
    L_R &= \bigO\left(\sum_{i=1}^K L_{R_i}\right), \qquad C_R = \bigO\!\left(\sum_{i=1}^K C_{R_i}\right) = \bigO\!\left(\sum_{i=1}^K L_{R_i} d D^3\right) = \bigO(L_R d D^3), \\
    M_R &= \max_i(M_{R_i}) = \bigO(\max_i(L_{R_i}) D^2) = \bigO(L_R D^2).
\end{align*}

\end{proof}

\section{Experimental Details}
\label{sec:appendix_d}

\begin{table}[h]
\caption{Definition of Tomita grammars 3-7 given in \citep{bengio1994}, which states a necessary and sufficient condition for a string to belong to each grammar. Tomita grammars 1 and 2 correspond to the respective family of strings $1^n$ and $(01)^*$, and are unused because of their small size and simple structure.}
\label{tab:tomita_def}
\begin{center}
\begin{small}
\begin{tabular}{c|lllll}
    \textbf{Tomita \#} & 3 & 4 & 5 & 6 & 7 \\
    \midrule
    \multirow{3}{*}{\textbf{Description}} & Doesn't contain         & Doesn't contain      & Contains an    & Number of 0's       & Has the form \\
                                          & $1^{2n+1}0^{2m+1}$ as a & $000$ as a substring & even number    & minus number of 1's & $0^*1^*0^*1^*$ \\
                                          & substring               &                      & of 0's and 1's & is a multiple of 3  & 
\end{tabular}
\end{small}
\end{center}
\end{table}

The u-MPS model we utilized was built from scratch in JAX~\citep{jax2018}, while the LSTM and Transformer modules from PyTorch~\citep{paszke2017} were used for baselines. The code for the experiments can be found at \url{https://github.com/jemisjoky/umps_code}. 
The LSTMs are single-layer models with 20 or 50 hidden units (20 or 50 in each direction for the bidirectional LSTM), and a linear decoder and softmax output layer used to obtain character probabilities. The bond dimension of the u-MPS was similar chosen to be 20 or 50, and for both types of models, five independent trials were used for each number of hidden states and the model with the lowest validation error was used to produce the sampling statistics reported in Section~\ref{sec:experiments}. 

For each grammar, the models were trained on either 1,000 or 10,000 randomly chosen strings, with 1,000 strings used as a held-out validation set. The sampling percentages for Table~\ref{tab:tomita} 
and the sampling tasks of Table~\ref{tab:motzkin} 
were obtained from sampling 1,000 random strings from the respective models, while the completion tasks of Table~\ref{tab:motzkin} 
used 1,000 random strings from a held-out reference set, where the models were used to infer each character in each string when all other characters were used as bidirectional context.

In all experiments, models were trained with gradient descent relative to a negative log likelihood (NLL) loss and Adam optimizer~\citep{kingma2015}. An initial learning rate of $10^{-2}$ was used, which was decreased by a factor of 10 each time the validation loss failed to improve for 5 consecutive epochs. In this manner, piecewise constant learning rates of $10^{-2}$, $10^{-3}$, and $10^{-4}$ were used, with the next drop in learning rate signalling the end of training.

The u-MPS, HMMs, and Transformers were trained identically for all experiments, with the unidirectional LSTM trained in the same way. The bidirectional LSTM was trained specifically for the string completion task, with the loss taken as a sum of the NLL of the correct character at each site of the training strings, given knowledge of all characters on the other sites. For each pair of sampling and completion tasks in Table~\ref{tab:motzkin}, 
the same trained u-MPS model was used to produce both statistics.

For the dataset of email addresses used in the real-text experiments, we extracted all sender and receiver addresses contained in the CLAIR fraudulent emails dataset~\citep{radev2008}, giving approximately 4,000 distinct addresses. Training was conducted in a similar manner to the synthetic data experiments, where the regular expression $R_{e} = \verb![\w-.]+@([\w-]+.)+[\w-][\w-]+!$ was used to judge the correctness of unconditionally sampled strings.

Our regex sampler and regex regularizer are straightforward recursive implementations of Algorithm~\ref{alg:sampling} and the correspondence in Table~\ref{tab:regex_cp}, respectively. Although these naive implementations would typically lead to a significant overhead compared to implementations specialized for $R_e$, the use of just-in-time (JIT) compilation within JAX gives a significant reduction in this overhead. Employing JIT in this setting is in fact historically well-motivated, considering that one of the earliest applications of JIT compilation was for identifying text matching regular expressions~\citep{thompson1968}. 

\bibliography{bib_file}

\end{document}